\documentclass[runningheads]{llncs}
\usepackage[ruled, vlined, nofillcomment, linesnumbered]{algorithm2e}
\usepackage{cite}
\usepackage{url}
\usepackage{amsmath}
\usepackage{amssymb}
\usepackage{graphicx}
\usepackage[caption=false]{subfig}
\usepackage[english]{babel}
\usepackage{booktabs}
\usepackage{nicefrac}
\usepackage{verbatim}
\usepackage{float}
\usepackage{xspace}
\usepackage{tikz}
\usepackage{pgfplots}
\usepackage[misc]{ifsym}
\usetikzlibrary{backgrounds}

\usepackage[numbers,compress]{natbib}

\renewcommand{\refname}{References}
\makeatletter
\renewcommand\bibsection{%
  \section*{{\refname}\@mkboth{\refname}{\refname}}%
}%
\makeatother

\SetKw{Output}{output}

\newcommand{\set}[1]{\left\{#1\right\}}
\newcommand{\pr}[1]{\left(#1\right)}
\newcommand{\fpr}[1]{\mathopen{}\left(#1\right)}
\newcommand{\spr}[1]{\left[#1\right]}

\newcommand{\brak}[1]{\left<#1\right>}
\newcommand{\abs}[1]{{\left|#1\right|}}

\newcommand{\define}{\leftarrow}

\DeclareRobustCommand{\dispfunc}[2]{%
	\ensuremath{%
		\ifthenelse{\equal{#2}{}}%
			{\mathit{#1}}%
			{\mathit{#1}\fpr{#2}}}}

\newcommand{\lhood}[1]{\dispfunc{\ell}{#1}}
\newcommand{\score}[1]{\dispfunc{q}{#1}}
\newcommand{\gain}[1]{\dispfunc{\Delta}{#1}}
\newcommand{\bigO}[1]{\dispfunc{\mathcal{O}}{#1}}

\newcommand{\freq}[1]{\dispfunc{av}{#1}}
\newcommand{\bernoulli}[1]{\dispfunc{Bern}{#1}}

\newcommand{\dtname}[1]{\textsl{#1}}

\newcommand{\findblocks}{\textsc{FindCands}\xspace}
\newcommand{\findseg}{\textsc{FindSegment}\xspace}
\newcommand{\findchange}{\textsc{FindChange}\xspace}

\newcommand{\prbchange}{\textsc{Change}\xspace}
\newcommand{\prbchangeinc}{\textsc{ChangeInc}\xspace}
\newcommand{\prbchangeblock}{\textsc{ChangeBlock}\xspace}

\pgfdeclarelayer{background}
\pgfdeclarelayer{foreground}
\pgfsetlayers{background,main,foreground}

\definecolor{yafaxiscolor}{rgb}{0.3, 0.3, 0.3}

\definecolor{yafcolor1}{rgb}{0.4, 0.165, 0.553}
\definecolor{yafcolor2}{rgb}{0.949, 0.482, 0.216}
\definecolor{yafcolor3}{rgb}{0.47, 0.549, 0.306}
\definecolor{yafcolor4}{rgb}{0.925, 0.165, 0.224}
\definecolor{yafcolor5}{rgb}{0.141, 0.345, 0.643}
\definecolor{yafcolor6}{rgb}{0.965, 0.933, 0.267}
\definecolor{yafcolor7}{rgb}{0.627, 0.118, 0.165}
\definecolor{yafcolor8}{rgb}{0.878, 0.475, 0.686}

\newlength{\yafaxispad}
\newlength{\yaftlpad}
\newlength{\yaflabelpad}
\newlength{\yafaxiswidth}
\newlength{\yafticklen}
\makeatletter
\def\pgfplots@drawtickgridlines@INSTALLCLIP@onorientedsurf#1{}
\makeatother

\newcommand{\yafdrawaxis}[4]{
	\pgfplotstransformcoordinatex{#1}\let\xmincoord=\pgfmathresult 
	\pgfplotstransformcoordinatex{#2}\let\xmaxcoord=\pgfmathresult 
	\pgfplotstransformcoordinatey{#3}\let\ymincoord=\pgfmathresult 
	\pgfplotstransformcoordinatey{#4}\let\ymaxcoord=\pgfmathresult 
	\pgfsetlinewidth{\yafaxiswidth} 
	\pgfsetcolor{yafaxiscolor}
	\pgfpathmoveto{\pgfpointadd{\pgfpointadd{\pgfplotspointrelaxisxy{0}{0}}{\pgfqpointxy{\xmincoord}{0}}}{\pgfqpoint{-0.5\yafaxiswidth}{\yafaxispad}}}
	\pgfpathlineto{\pgfpointadd{\pgfpointadd{\pgfplotspointrelaxisxy{0}{0}}{\pgfqpointxy{\xmaxcoord}{0}}}{\pgfqpoint{0.5\yafaxiswidth}{\yafaxispad}}}
	\pgfpathmoveto{\pgfpointadd{\pgfpointadd{\pgfplotspointrelaxisxy{0}{0}}{\pgfqpointxy{0}{\ymincoord}}}{\pgfqpoint{\yafaxispad}{-0.5\yafaxiswidth}}}
	\pgfpathlineto{\pgfpointadd{\pgfpointadd{\pgfplotspointrelaxisxy{0}{0}}{\pgfqpointxy{0}{\ymaxcoord}}}{\pgfqpoint{\yafaxispad}{0.5\yafaxiswidth}}}
	\pgfusepath{stroke}
}

\pgfplotscreateplotcyclelist{yaf}{%
{yafcolor5,mark options={scale=0.75},mark=o}, 
{yafcolor2,mark options={scale=0.75},mark=square},
{yafcolor3,mark options={scale=0.75},mark=triangle},
{yafcolor4,mark options={scale=0.75},mark=o},
{yafcolor1,mark options={scale=0.75},mark=o},
{yafcolor6,mark options={scale=0.75},mark=o},
{yafcolor7,mark options={scale=0.75},mark=o},
{yafcolor8,mark options={scale=0.75},mark=o}} 

\pgfkeys{/pgf/number format/.cd,1000 sep={\,}}

\pgfplotsset{axis y line=left, axis x line=bottom,
	tick align=outside,
	tickwidth=\yafticklen,
	clip = false,
    x axis line style= {-, line width = 0pt, color=black!0},
    y axis line style= {-, line width = 0pt, color=black!0},
    x tick style= {line width = \yafaxiswidth, color=yafaxiscolor, yshift = \yafaxispad},
    y tick style= {line width = \yafaxiswidth, color=yafaxiscolor, xshift = \yafaxispad},
    x tick label style = {font=\scriptsize, yshift = \yaftlpad, inner xsep = 0pt},
    y tick label style = {font=\scriptsize, xshift = \yaftlpad},
    every axis y label/.style = {at = {(ticklabel cs:0.5)}, rotate=90, anchor=center, font=\scriptsize, yshift = -\yaflabelpad, inner sep = 0pt},
    every axis x label/.style = {at = {(ticklabel cs:0.5)}, anchor=center, font=\scriptsize, yshift = \yaflabelpad},
    x tick label style = {font=\scriptsize, yshift = 1pt},
    grid = major,
    major grid style  = {dash pattern = on 1pt off 3 pt},
	every axis plot post/.append style= {line width=\yafaxiswidth} ,
	legend cell align = left,
	legend style = {inner sep = 1pt, cells = {font=\scriptsize}},
	legend image code/.code={%
		\draw[mark repeat=2,mark phase=2,#1] 
		plot coordinates { (0cm,0cm) (0.15cm,0cm) (0.3cm,0cm) };%
	} 
}

\begin{document}

\title{Fast likelihood-based change point detection}
%
%

\author{Nikolaj Tatti\orcidID{0000-0002-2087-5360}}
\authorrunning{N. Tatti}
\tocauthor{Nikolaj Tatti}
\toctitle{Fast likelihood-based change point detection}

\institute{HIIT, University of Helsinki, Helsinki, Finland\\
\email{nikolaj.tatti@helsinki.fi}}

\maketitle              

\begin{abstract}
Change point detection plays a fundamental role in many real-world applications,
where the goal is to analyze and monitor the behaviour of a data stream.
In this paper, we study change detection in binary streams.  To this end, we
use a likelihood ratio between two models as a measure for indicating
change. The first model is a single bernoulli variable while the second model
divides the stored data in two segments, and models each segment with its
own bernoulli variable. Finding the optimal split can be done in $\bigO{n}$ time,
where $n$ is the number of entries since the last change point.
This is too expensive for large $n$. To combat this we propose an approximation
scheme that yields $(1 - \epsilon)$ approximation in $\bigO{\epsilon^{-1} \log^2 n}$ time.
The speed-up consists of several steps: First we reduce the number of possible
candidates by adopting a known result from segmentation problems. We then show
that for fixed bernoulli parameters we can find the optimal change point in
logarithmic time.  Finally, we show how to construct a candidate list
of size $\bigO{\epsilon^{-1} \log n}$ for model parameters.
We demonstrate empirically the approximation quality and the running time of
our algorithm, showing that we can gain a significant speed-up with a minimal
average loss in optimality.

\end{abstract}

\section{Introduction}
Many real-world applications involve in monitoring and analyzing a constant
stream of data. A fundamental task in such applications is to monitor whether
a change has occurred.  For example, the goal may be monitoring the performance
of a classifier over time, and triggering retraining if the quality degrades
too much. We can also use change point detection techniques to detect anomalous
behavior in the data stream. As the data flow may be significant, it is important
to develop efficient algorithms.

In this paper we study detecting change in a stream of binary numbers, that is,
we are interested in detecting whether the underlying distribution has recently changed
significantly.  To test the change we will use a standard likelihood ratio
statistic. Namely, assume that we have already observed $n$ samples from the last time
we have observed change. In our
first model, we fit a single bernoulli variable to these samples. In our second
model, we split these samples in two halves, say at point $i$, and fit
two bernoulli variables to these halves. Once this is done we compare the
likelihood ratio of the models. If the ratio is large enough, then we deem that
change has occurred. 

In our setting, index $i$ is not fixed. Instead we are looking for the index
that yields the largest likelihood. This can be done naively in $\bigO{n}$ time
by testing each candidate. This may be too slow, especially if $n$ is large
enough and we do not have the resources before a new sample arrives.
Our main technical contribution is to show how
we can achieve $(1 - \epsilon)$ approximate of the optimal $i$ in $\bigO{\epsilon^{-1} \log^2 n}$ time.

To achieve this we will first reduce the number of candidates for the optimal
index $i$. We say that index $j$ is a \emph{border} if each interval ending at $j -
1$ has a smaller proportion of 1s that any interval that starts at $j$.  A known result
states that the optimal change point will be among border indices.  Using
border indices already reduces the search time greatly in practice, with
theoretical running time being $\bigO{n^{2/3}}$.

To obtain even smaller bounds we show that we can find the optimal index among the
border indices for \emph{fixed} model parameters, that is, the parameters for
the two bernoulli variables, in $\bigO{\log n}$ time. We then construct a list
of $\bigO{\epsilon^{-1} \log n}$ candidates for these parameters.  Moreover, this
list will contain model parameters that are close enough to the optimal 
parameters, so testing them yields $(1 - \epsilon)$ approximation
guarantee in $\bigO{\epsilon^{-1} \log^2 n}$ time.

The remaining paper is organized as follows. In Section~\ref{sec:prel} we
introduce preliminary notation and define the problem. In
Section~\ref{sec:block} we introduce border points. We present our main
technical contribution in Sections~\ref{sec:segment}--\ref{sec:frequency}:
first we show how to find optimal index for fixed model parameters, and then
show how to select candidates for these parameters. We present related work
in Section~\ref{sec:related} and empirical evaluation in Section~\ref{sec:exp}.
Finally, we conclude with discussion in Section~\ref{sec:conclusions}.

\section{Preliminaries and problem definition}\label{sec:prel}

Assume a sequence of $n$ binary numbers $S = s_1, \ldots, s_n$.
Here $s_1$ is either the beginning of the stream or the last time we detected a change.
Our goal is to determine whether a change
has happened in $S$. More specifically, we consider two statistical models:
The first model $M_1$ assumes that $S$ is generated with a single bernoulli variable.
The second model $M_2$ assumes that there is an index $i$, a change point, such that
$s_1, \ldots, s_{i - 1}$ is generated by one bernoulli variable
and
$s_i, \ldots, s_n$ is generated by another bernoulli variable.

Given a sequence $S$ we will fit $M_1$ and $M_2$ and compare
the log-likelihoods. Note that the model $M_2$ depends on the change point $i$, so we
need to select $i$ that maximizes the likelihood of $M_2$.  If the ratio is large
enough, then we can determine that change has occurred.

To make the above discussion more formal, let us introduce some notation.
Given two integers $a$ and $b$, and real number between 0 and 1, we denote the
log-likelihood of a bernoulli variable by
\[
	\lhood{a, b; p} = a \log p + b \log ( 1 - p)\quad.
\]
For a fixed $a$ and $b$, the log-likelihood is at its maximum if $p = a / (a + b)$.
In such a case, we will often drop $p$ from the notation and simply write $\lhood{a, b}$.

We have the following optimization problem.

\begin{problem}[\prbchange]
Given a sequence $S = s_1, \ldots, s_n$, find an index $i$ s.t.
\[
	\lhood{a_1, b_1} + \lhood{a_2, b_2} - \lhood{a, b}
\]
is maximized, where
\[
\begin{split}
	a_1 & = \sum_{j = 1}^{i - 1} s_i, \ 
	b_1  = i - 1 - a_1, \quad
	a_2  = \sum_{j = i}^{k} s_i, \ 
	b_2  = k - i - a_2, \\
	a & = a_1 + a_2, \text{ and} \ 
	b  = b_1 + b_2\quad.
\end{split}
\]
\end{problem}

Note that \prbchange can be solved in $\bigO{n}$ time by simply iterating over all possible values for $i$.
Such running time may be too slow, especially in a streaming setting when new points arrive constantly,
and our goal is to determine whether change has occurred in real time.
The main contribution of this paper is to show how to compute $(1 - \epsilon)$ estimate of \prbchange
in $\bigO{\epsilon^{-1} \log^2 n}$ time. This algorithm requires additional data structures
that we will review in the next section. 
As our main application is to search change points in a stream, these structures
need to be maintained over a stream. Luckily, there is an
amortized constant-time algorithm for maintaining the needed structure, as demonstrated in the next section.

Once we have solved \prbchange, we compare the obtained score against the
threshold $\sigma$. Note that $M_2$ will always have a larger likelihood than $M_1$. In
this paper, we will use BIC to adjust for the additional model complexity of
$M_2$. The model $M_2$ has three parameters while the model $M_1$ has 1 parameter.
This leads to a BIC penalty of $(3 - 1)/2 \log n = \log n$. In practice, we need to
be more conservative when selecting $M_2$ due to the multiple hypothesis testing problem.
Hence, we will use $\sigma = \tau + \log n$ as the threshold. Here, $\tau$ is a user parameter; we will provide
some guidelines in selecting $\tau$ during the experimental evaluation in Section~\ref{sec:exp}.

When change occurs at point $i$ we have two options: we can either discard
the current window and start from scratch, or we can drop only the first $i$ elements.
In this paper we will use the former approach since the latter approach requires
additional maintenance which may impact overall computational complexity.

\section{Reducing number of candidates}\label{sec:block}

Our first step for a faster change point discovery is to reduce the number of
possible change points. To this end, we define a variant of \prbchange, where we
require that the second parameter in $M_2$ is larger than the first.

\begin{problem}[\prbchangeinc]
Given a sequence $S = s_1, \ldots, s_n$, find an index $i$ s.t.
\[
    \lhood{a_1, b_1} + \lhood{a_2, b_2} - \lhood{a, b}
\]
is maximized, where
\[
\begin{split}
    a_1 & = \sum_{j = 1}^{i - 1} s_i, \
    b_1  = i - 1 - a_1, \quad
    a_2  = \sum_{j = i}^{k} s_i, \
    b_2  = k - i - a_2, \\
    a & = a_1 + a_2, \text{ and} \
    b  = b_1 + b_2
\end{split}
\]
with $a_1 / (a_1 + b_1) \leq a_2 / (a_2 + b_2)$.
\end{problem}

From now on, we will focus on solving \prbchangeinc. This problem is meaningful
by itself, for example, if the goal is to detect a deterioration in a classifier,
that is, sudden increase in entries being equal to 1. However, we can also use
\prbchangeinc to solve \prbchange. This is done by defining a flipped sequence
$S' = s'_1, \ldots, s'_n$, where $s'_i = 1 - s_i$. Then the solution for
\prbchange is either the solution of $\prbchangeinc(S)$ or the solution  of
$\prbchangeinc(S')$.

Next we show that we can limit ourselves to \emph{border} indices when solving \prbchangeinc.

\begin{definition}
Assume a sequence of binary numbers $S = (s_i)_{i = 1}^n$. We say that index $j$ is a \emph{border} index
if there are no indices $x, y$ with $x < j < y$ such that
\[
	\frac{1}{j - x}\sum_{i = x}^{j - 1} s_i \geq \frac{1}{y - j}\sum_{i = j}^{y - 1} s_i \quad. 
\]
\end{definition}
In other words, $j$ is a border index if and only if the average of any interval ending at $j - 1$ is smaller
than the average of any interval starting at $j$.

\begin{proposition} 
There is a border index $i$ that solves \prbchangeinc.
\label{prop:block}
\end{proposition}

The proposition follows from a variant of Theorem~1~in~\citep{tatti2013fast}.
For the sake of completeness we provide a direct proof in Appendix in
supplementary material.

We address the issue of maintaining border indices at the end of this section.

The proposition permits us to ignore all indices that are not borders.
That is, we can group the sequence entries in blocks, each block starting with a
border index. We can then search for $i$ using these blocks instead
of using the original sequence. 

It is easy to see that these blocks have the following property: the proportion
of 1s in the next block is always larger.  This key feature will play a crucial
role in the next two sections as it allows us to use binary search techniques
and reduce the computational complexity.
Let us restate the original problem so that we can use this feature.
First, let us define what is a block sequence.

\begin{definition}
Let $B = \brak{(u_i, v_i)}_{i = 1}^k$ be a sequence of $k$ pairs of non-negative integers with $u_i + v_i > 0$.
We say that $B$ is \emph{block sequence} if $\frac{u_{i + 1}}{u_{i + 1} + v_{i + 1}} > \frac{u_{i}}{u_{i} + v_{i}}$.
\end{definition}

We obtain a block sequence $B$ from a binary sequence $S$ by grouping the entries between border points: the counter $u_i$
indicates the number of 1s while the counter $v_i$ indicates the number of 0s.

Our goal is to use block sequences to solve \prbchangeinc. First, we need some additional notation.

\begin{definition}
Given a block sequence $B$, we define $B[i;j] = (a, b)$,
where $a = \sum_{k = i}^j u_k$ and $b = \sum_{k = i}^j v_k$.
If $i > j$, then $a = b = 0$. Moreover, we will write
\[
	\freq{i, j; B} = \frac{a}{a + b}\quad.
\]
If $B$ is known from the context, we will write
$\freq{i, j}$.
\end{definition}

\begin{definition}
Given a block sequence $B$, we define the score of a change point $i$ to be
\begin{equation}
\label{eq:blockscore}
	\score{i; B} = \lhood{a_1, b_1} + \lhood{a_2, b_2} - \lhood{a, b},
\end{equation}
where $(a_1, b_1) = B[1;i - 1]$, $(a_2, b_2) = B[i;k]$, and $a = a_1 + a_2$ and $b = b_1 + b_2$.
\end{definition}
Note that $\lhood{a, b}$ is a constant but it is useful to keep since $\score{i; B}$  is
a log-likelihood ratio between two models, and this formulation allows us to estimate the objective
in Section~\ref{sec:frequency}. 

\begin{problem}[\prbchangeblock]
Given a block sequence $B$ find a change point $i$ that maximizes $\score{i; B}$.
\end{problem}

We can solve \prbchangeinc by maintaining a block sequence induced by the border points,
and solving \prbchangeblock. Naively, we can simply compute $\score{i; B}$
for each index in $\bigO{\abs{B}}$ time. If the distribution is static, then $\abs{B}$
will be small in practice. However, if there is a concept drift, that is, there are more
1s in the sequence towards the end of sequence, then $\abs{B}$ may increase significantly.
\citet{calders2008pava} argued that when dealing with binary sequences of length $n$, the number of blocks $\abs{B} \in \bigO{n^{2/3}}$.
In the following two sections we will show how to solve \prbchangeblock faster.

However, we also need to maintain the block sequence as new entries arrive.
Luckily, there is an efficient update algorithm, see~\cite{calders2008pava} for example.
Assume that we have already observed $n$ entries, and we have a block sequence
of $k$ blocks $B$ induced by the border points. 
Assume a new entry $s_{n + 1}$. We add $(k + 1)$th block $(u_{k + 1}, v_{k + 1})$ to $B$,
where $u_{k + 1} = [s_{n + 1} = 1]$ and $v_{k + 1} = [s_{n + 1} = 0]$.
We then check whether $\freq{k + 1, k + 1} \leq \freq{k, k}$, that is, whether the average
of the last block is smaller than or equal to the average of the second last block.
If it is, then we merge the blocks and repeat the test. This algorithm maintains the border
points correctly and runs in amortized $\bigO{1}$ time.

It is worth mentioning that the border indices are also connected to isotonic
regression (see~\cite{leeuw2009isotonic}, for example). Namely, if one would
fit isotonic regression to the sequence $S$, then the border points are the
points where the fitted curve changes its value.  In fact, the update algorithm
corresponds to the  pool adjacent violators (PAVA) algorithm, a method used to
solve isotonic regression~\cite{leeuw2009isotonic}.

\section{Finding optimal change point for fixed parameters}\label{sec:segment}

In this section we show that if the model parameters are known and fixed, then we can find
the optimal change point in logarithmic time.

First, let us extend the definition of $\score{\cdot}$ to handle fixed parameters.
\begin{definition}
Given a block sequence $B$, an index $i$, and two parameters $p_1$ and $p_2$, we define
\[
    \score{i; p_1, p_2, B} = \lhood{a_1, b_1; p_1} + \lhood{a_2, b_2; p_2} - \lhood{a, b},
\]
where $(a_1, b_1) = B[1;i - 1]$, $(a_2, b_2) = B[i;k]$, and $a = a_1 + a_2$ and $b = b_1 + b_2$.
\end{definition}

We can now define the optimization problem for fixed parameters.

\begin{problem}
\label{prb:maxscorepar}
Given a block sequence $B$, two parameters $0 \leq p_1 < p_2 \leq 1$, find $i$ maximizing $\score{i; p_1, p_2, B}$.
\end{problem}

Let $i^*$ be the solution for Problem~\ref{prb:maxscorepar}.
It turns out that we can construct a sequence of numbers, referred as $d_j$ below, such that
$d_j > 0$ if and only if $j < i^*$. This allows us to use binary search to find $i^*$.

\begin{proposition}
\label{prop:maxscorepar}
Assume a block sequence $B = \brak{(u_j, v_j)}$ and two parameters $0 \leq p_1 < p_2 \leq 1$. Define
\[
	d_j = \lhood{u_j, v_j, p_1}  - \lhood{u_j, v_j, p_2}\quad.
\]
Then there is an index $i$ such that $d_j > 0$ if and only if $j < i$. Moreover, index $i$ solves Problem~\ref{prb:maxscorepar}.
\end{proposition}

\begin{proof}
Let us first show the existence of $i$. Let $t_j = u_j + v_j$, and
write $X = \log p_1 - \log p_2$ and $Y = \log (1 - p_1) - \log (1 - p_2)$.
Then
\[
	\frac{d_j}{t_j} = \frac{u_j}{t_j} X + \frac{v_j}{t_j} Y = \frac{u_j}{t_j} X + Y - \frac{u_j}{t_j} Y = \frac{u_j}{t_j} (X - Y) + Y\quad.
\]
Since $B$ is a block sequence, the fraction $u_j / t_j$ is increasing. Since $X < 0$ and $Y > 0$, we have $X - Y < 0$, so $d_j / t_j$ is decreasing.
Since $d_j$ and $d_j / t_j$ have the same sign, there is an index $i$ satisfying the condition of the statement.

To prove the optimality of $i$, first note that
\[
	d_j = \score{j + 1; p_1, p_2, B} -  \score{j; p_1, p_2, B}\quad.
\]
Let $i^*$ be a solution for Problem~\ref{prb:maxscorepar}. If $i < i^*$. Then
\[
	 \score{i^*; p_1, p_2, B} -  \score{i; p_1, p_2, B} = \sum_{j = i}^{i^* - 1} d_j \leq 0,
\]
proving the optimality of $i$. The case for $i > i^*$ is similar.\qed
\end{proof}

Proposition~\ref{prop:maxscorepar} implies that we can use binary search to
solve Problem~\ref{prb:maxscorepar} in $\bigO{\log \abs{B}} \in \bigO{\log n}$
time. We refer to this algorithm as $\findseg(p_1, p_2, B)$.

\section{Selecting model parameters}\label{sec:frequency}

We have shown that
if we know the optimal $p_1$ and $p_2$, then we can use binary search as described in the previous
section to find the change point. 
Our main idea is to test several candidates for $p_1$ and $p_2$ such that one of the candidates
will be close to the optimal parameters yielding an approximation guarantee.

Assume that we are given a block sequence $B$ and select a change point $i$.
Let $(a_1, b_1) = B[1; i - 1]$, $(a_2, b_2) = B[i; k]$, $a = a_1 + a_2$, $b = b_1 + b_2$
be the counts.
We can rewrite objective given in Eq.~\ref{eq:blockscore} as 
\begin{equation}
\label{eq:split}
\begin{split}
	\score{i; B} & = \lhood{a_1, b_1} + \lhood{a_2, b_2} - \lhood{a, b} \\
	             & = \pr{\lhood{a_1, b_1, p_1} - \lhood{a_1, b_1, q}} + \pr{\lhood{a_2, b_2, p_2} - \lhood{a_2, b_2, q}},
\end{split}
\end{equation}
where the model parameters are $p_1 = a_1 / (a_1 + b_1)$,
$p_2 = a_2 / (a_2 + b_2)$, and $q = a / (a + b)$.

The score as written in Eq.~\ref{eq:split} is split in two parts, the first
part depends on $p_1$ and the second part depends on $p_2$. We will first focus
solely on estimating the second part. First, let us show how much we can vary
$p_2$ while still maintaining a good log-likelihood ratio.

\begin{proposition}
\label{prop:bracket}
Assume $a, b > 0$, and let $p = a / (a + b)$.
Assume $0 < q \leq p$.
Assume also $\epsilon > 0$.  Define $h(x) = \lhood{a, b; x} - \lhood{a, b; q}$.
Assume $r$ such that
\begin{equation}
	\label{eq:bracket}
	\log q + (1 - \epsilon) (\log p - \log q) \leq \log r \leq \log p\quad. 
\end{equation}
Then $h(r) \geq (1 - \epsilon) h(p)$.
\end{proposition}

\begin{proof} 
Define $f(u) = h(\exp u)$. We claim that $f$ is concave. To prove the claim,
note that the derivative of $f$ is equal to
\[
    f'(u) = a - b\frac{\exp u}{1 - \exp u}\quad.
\]
Hence, $f'$ is decreasing for $u < 0$, which proves the concavity of $f$.

Define $c = \frac{\log r - \log q}{ \log p - \log q}$. Eq.~\ref{eq:bracket}
implies that $1 - \epsilon \leq c$.
The concavity of $f(u)$ and the fact that $h(q) = 0$ imply that
\[
	h(r) = f(\log r) \geq f(\log q) + c\spr{f(\log p) - f(\log q)} = ch(p) \geq (1 - \epsilon)h(p),
\]
which proves the proposition.\qed
\end{proof}

We can use the proposition in the following manner.
Assume a block sequence $B$ with $k$ entries. Let $i^*$ be the optimal change point
and $p_1^*$ and $p_2^*$ be the corresponding optimal parameters. 
First, let 
\[
	P = \set{ \freq{i, k} \mid i = 1, \ldots, k}
\]
be the set of candidate model parameters. We
know that the optimal model parameter $p_2^* \in P$. Instead of testing every $p
\in P$, we will construct an index set $C$, and define $R = \set{ \freq{i, k} \mid i \in C}$,
such that
for each $p \in P$ there is $r \in R$ such that Eq.~\ref{eq:bracket} holds.
Proposition~\ref{prop:bracket} states that testing the parameters in $R$
yields a $(1 - \epsilon)$ approximation of the second part of the right-hand side in Eq.~\ref{eq:split}.

We wish to keep the set $C$ small, so to generate $C$, we will start with
$i = 1$ and set $C = \set{i}$.  We then look how many values of $P$ we can estimate with 
$\freq{i, k}$, that is, we look for the smallest index for which
Eq.~\ref{eq:bracket} does not hold. We set this index to $i$, add it to $C$, and repeat the process.
We will refer to this procedure as $\findblocks(B, \epsilon)$. 
The detailed pseudo-code for \findblocks is given in Algorithm~\ref{alg:findcands}.

\begin{algorithm}
\caption{$\findblocks(B, \epsilon)$, given a block sequence $B$ of $k$ entries and an estimation requirement
$\epsilon > 0$,
constructs a candidate index set $C$ that is used to estimate
the model parameter $p_2$.}
\label{alg:findcands}
$C \define \set{1}$;
$i \define 1$;
$q \define \freq{1, k}$\;
\While {$i < k$} {
	$\rho \define (\log \freq{i, k} - \log q) / (1 - \epsilon)$\;
	$i \define $ smallest index $j$ s.t. $\log \freq{j, k} - \log q > \rho$, or $k$ if $j$ does not exist\;
	add $i$ to $C$;
}
\Return $C$\;
\end{algorithm}

\begin{proposition}
Assume a block sequence $B$ with $k$ entries, and let $\epsilon > 0$. Set $P =
\set{ \freq{i, k} \mid i = 1, \ldots, k}$.
Let $C = \findblocks(B, \epsilon)$, and let $R = \set{ \freq{i, k} \mid i \in C}$.
Then for each $p \in P$ there is $r \in R$ such that Eq.~\ref{eq:bracket} holds.
\end{proposition}
\begin{proof}
Let $p \in P \setminus R$. This is only possible if there is a smaller value $r \in R$ such that
$(1 - \epsilon)(\log p - \log q) < \log r - \log q$ holds. 
\qed
\end{proof}

Finding the next index $i$ in \findblocks can be done with a binary search in $\bigO{\log \abs{B}}$ time.
Thus, \findblocks runs in $\bigO{\abs{C} \log n}$ time. Next result shows that $\abs{C} \in \bigO{\epsilon^{-1} \log n}$,
which brings the computational complexity of \findblocks to $\bigO{\epsilon^{-1} \log^2 n}$.

\begin{proposition}
\label{prop:candsize}
Assume a block sequence $B$ with $k$ entries generated from a binary sequence $S$ with $n$ entries, and let $\epsilon > 0$.
Let $P = \set{ \freq{i, k} \mid i = 1, \ldots, k}$.
Assume an increasing sequence $R = (r_i) \subseteq P$.
Let $q = \freq{1, k}$.
If
\begin{equation}
\label{eq:ladder}
	\log q + (1 - \epsilon) (\log r_{i} - \log q) > \log r_{i - 1},
\end{equation}
then $\abs{R} \in \bigO{\frac{\log n}{\epsilon}}$.
\end{proposition}

\begin{proof}
We can rewrite Eq.~\ref{eq:ladder} as
	$(1 - \epsilon) (\log r_{i} - \log q) > \log r_{i - 1} - \log q$
which automatically implies that
\[
	(1 - \epsilon)^i (\log r_{i + 2} - \log q) > \log r_{2} - \log q\quad.
\]
To lower-bound the right-hand side, let us write $r_2 = x / y$ and $q = u / v$,
where $x$, $y$, $u$, and $v$ are integers with $y, v \leq n$. Note that $r_2 >
q$, otherwise we violate Eq.~\ref{eq:ladder} when $i = 2$. Hence, we have $xv \geq uy + 1$. Then
\[
\begin{split}
	\log r_2 - \log q & = \log xv - \log uy 
	                  \geq \log(uy + 1) - \log uy = \log (1 + \frac{1}{uy}) \\
					  & \geq \log(1 + \frac{1}{n^2}) \geq \frac{n^{-2}}{1 + n^{-2}} = \frac{1}{1 + n^{2}}\quad.
\end{split}
\]
We can also upper-bound the left-hand side with
\[
	\log r_{i + 2} - \log q \leq \log 1 - \log u / v = \log v / u \leq \log n\quad.
\]
Combining the three previous inequalities leads to
\[
	\log n \geq \log r_{i + 2} - \log q > \frac{\log r_{2} - \log q}{ (1 - \epsilon)^i} \geq \frac{1}{(1 - \epsilon)^i}\frac{1}{1 + n^{2}}\quad.
\]
Solving for $i$,
\[
	i \leq \frac{\log(1 + n^2) + \log \log n}{\log \frac{1}{1 - \epsilon} } \leq \frac{\log(1 + n^2) + \log \log n}{\epsilon } \in \bigO{\frac{\log n}{\epsilon}},
\]
completes the proof.\qed
\end{proof}

We can now approximate $p_2^*$. Our next step is to show how to find similar value for $p_1^*$. Note that
we cannot use the previous results immediately because we assumed that $p \geq
q$ in Proposition~\ref{prop:bracket}. However, we can fix this by simply switching the labels in $S$.

\begin{proposition}
\label{prop:bracket2}
Assume $a, b > 0$, and let $p = a / (a + b)$.
Assume $q$ with $0 < p \leq q$.
Assume also $\epsilon > 0$.  Define $h(x) = \lhood{a, b; x} - \lhood{a, b; q}$.
Assume $r$ such that
\begin{equation}
	\label{eq:bracket2}
	\log (1 - q) + (1 - \epsilon) (\log(1 - p) - \log(1 - q)) \leq \log(1 - r) \leq \log(1 - p)\quad. 
\end{equation}
Then $h(r) \geq (1 - \epsilon) h(p)$.
\end{proposition}

\begin{proof}
Set $a' = b$, $b' = a$, $q' = 1 - q$, and $r' = 1 - r$.
The proposition follows immediately from Proposition~\ref{prop:bracket} when applied to these variables.\qed
\end{proof}

Proposition~\ref{prop:bracket2} leads to an algorithm, similar to \findblocks, for generating
candidates for $p_1^*$. We refer to this algorithm as $\findblocks'$, see Algorithm~\ref{alg:findcands2}. 

\begin{algorithm}[ht!]
\caption{$\findblocks'(B, \epsilon)$, given a block sequence $B$ of $k$ entries and an estimation requirement
$\epsilon > 0$,
constructs a candidate index set $C$ that is used to estimate
the model parameter $p_1$.}
\label{alg:findcands2}
$C \define \set{k}$;
$i \define k$;
$q \define \freq{1, k}$\;
\While {$i > 1$} {
	$\rho \define (\log(1 - \freq{1, i - 1}) - \log (1 - q)) / (1 - \epsilon)$\;
	$i \define $ largest index $j$ s.t. $\log ( 1- \freq{1, j - 1}) - \log (1 - q) > \rho$, or $1$ if $j$ does not exist\;
	add $i$ to $C$;
}
\Return $C$\;
\end{algorithm}

Assume that we have computed two sets of candidate indices $C_1$ and $C_2$; the
first set is meant to be used to estimate $p_1^*$, while the second set is meant
to be used to estimate $p_2^*$. The final step is to determine what combinations
of parameters should we check. A naive approach would be to test every possible
combination. This leads to $\bigO{\abs{C_1}\abs{C_2}}$ tests.

However, since $p_1^*$ and $p_2^*$ are induced by the same change point $i^*$, we
can design
a more efficient approach that leads to only $\bigO{\abs{C_1} + \abs{C_2}}$ tests.
In order to do so, first we combine both candidate sets, $C = C_1 \cup C_2$.
For each index $c_i \in C$, we compute the score $\score{c_i; B}$.
Also, if there are blocks between $c_{i - 1}$ and $c_i$ that are not
included in $C$, that is, $c_{i - 1} + 1 < c_i$, we set $p_1 = \freq{1, c_{i} - 1}$
and $p_2 = \freq{c_{i - 1}, k}$, compute the optimal change point
$j = \findseg(p_1, p_2, B)$, and test $\score{j, B}$. When all tests are done,
we return the index that yielded the best score. We refer to this algorithm
as $\findchange(B, \epsilon)$, and present the pseudo-code
in Algorithm~\ref{alg:change}.

\begin{proposition}
\label{prop:correct}
$\findchange(B, \epsilon)$ yields $(1 - \epsilon)$ approximation guarantee.
\end{proposition}

\begin{proof}
Let $i^*$ be the optimal value with the corresponding parameters $p_1^*$ and
$p_2^*$.  Let $C_1$, $C_2$ and $C$ be the sets as defined in
Algorithm~\ref{alg:change}. If $i^* \in C$, then we are done.  Assume that $i^*
\notin C$. Then there are $c_{j - 1} < i^* < c_j$, since $1, k \in C$.  Let $r_2
= \freq{c_{j - 1},k}$. Then $r_2$ and $p_2^*$ satisfy Eq.~\ref{eq:bracket} by
definition of $C_2$.  Let $r_1 = \freq{1, c_{j} - 1}$. Then $r_1$ and $p_1^*$
satisfy Eq.~\ref{eq:bracket2} by definition of $C_1$. Let $i$ be the optimal
change point for $r_1$ and $r_2$, that is, $i = \findseg(r_1, r_2, B)$.

Propositions~\ref{prop:bracket}~and~\ref{prop:bracket2} together with Eq.~\ref{eq:split}
imply that
\[
	\score{i; B} \geq \score{i; r_1, r_2, B} \geq \score{i^*; r_1, r_2, B} \geq (1 - \epsilon) \score{i^*; B}\quad.
\]
This completes the proof.\qed
\end{proof}

\begin{algorithm}
\caption{$\findchange(B, \epsilon)$, yields $(1 - \epsilon)$ approximation guarantee for \prbchangeblock.}
\label{alg:change}
$C_2 \define$ $\findblocks(B, \epsilon)$\;
$C_1 \define $ $\findblocks'(B, \epsilon)$\;
$C \define C_1 \cup C_2$\;

\ForEach{$c_j \in C$} {
	test $\score{c_j; B}$\;
	\If {$c_{j - 1} + 1 < c_j$} {
		$r_1 \define \freq{1,c_{j} - 1}$\;
		$r_2 \define \freq{c_{j - 1},k}$\;
		$i \define \findseg(r_1, r_2, B)$\;
		test $\score{i; B}$\;
	}
	\Return index $i^*$ having the best score $\score{i; B}$ among the tested indices\;
}

\end{algorithm}

We complete this section with computational complexity analysis. The two calls
of \findblocks require $\bigO{\epsilon^{-1} \log^2 n}$ time. The list $C$ has
$\bigO{\epsilon^{-1} \log n}$ entries, and a single call of \findseg for each $c \in C$ requires $\bigO{\log n}$ time.
Consequently, the running time for \findchange is $\bigO{\epsilon^{-1} \log^2 n}$.

\section{Related work}\label{sec:related}

Many techniques have been proposed for change detection in a stream setting. We will
highlight some of these techniques.  For a fuller picture, we refer the reader
to a survey by~\citet{Aminikhanghahi2017}, and a book by~\citet{Basseville93}.

A standard approach for change point detection is to split the stored data
in two segments, and compare the two segments; if the segments are different, then
a change has happened.
\citet{Bifet07} proposed an adaptive sliding window approach: if the current window
contains a split such that the averages of the two portions are different enough,
then the older portion is dropped from the window.
\citet{Nishida07} compared the accuracy of recent samples against the overall accuracy
using a statistical test.
\citet{Kifer04} proposed a family of distances between distributions and analyzed them
in the context of change point detection.
Instead of modeling segments explicitly, \citet{Kawahara12} proposed
estimating density ratio directly. 
\citet{Dries09} studied transformations a multivariate stream into a
univariate stream to aid change point detection.
\citet{Harel14} detected change by comparing the loss in a test segment against
a similar loss in a permuted sequence.

Instead of explicitly modeling the change point, \citet{Ross12} used
exponential decay to compare the performance of recent samples against the
overall performance. \citet{Gama04,Baena06} proposed a detecting change by
comparing current average and standard deviation against the smallest observed
average and standard deviation.  Also avoiding an explicit split, a Bayesian
approach for modeling the time since last change point was proposed
by~\citet{Adams07}.

An offline version of change point detection is called \emph{segmentation}.
Here we are given a sequence of entries and a budget $k$. The goal is divide a
sequence into $k$ minimizing some cost function. If the global objective is a
sum of individual segment costs, then the problem can be solved with a classic
dynamic program approach~\citep{bellman:61:on} in $\bigO{n^2k}$ time. As this
may be too slow speed-up techniques yielding approximation guarantees have been
proposed~\citep{terzi:06:efficient,guha:06:estimate,tatti2019segmentation}.
If the cost function is based on one-parameter
log-linear models, it is possible to speed-up the segmentation problem
significantly in practice~\citep{tatti2013fast}, even though the worst-case running time
remains $\bigO{n^2 k}$. 
\citet{guha:07:linear} showed that if the objective is the maximum of the
individual segment costs, then
we can compute the exact solution using only $\bigO{k^2 \log^2 n}$ evaluations of the individual segment costs.

\section{Experimental evaluation}\label{sec:exp}

For our experiments, we focus on analyzing the effect
of the approximation guarantee $\epsilon$, as well as 
the parameter $\tau$.\footnote{Recall that we say that change occurs if it is larger than $\sigma = \tau + \log n$.}
\footnote{The implementation is available at \url{https://version.helsinki.fi/dacs/}.}
Here we will use synthetic sequences.  In addition, we present a small case
study using network traffic data.

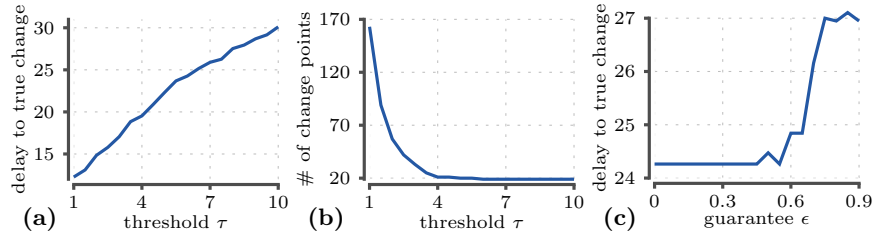
\begin{figure}
\subfloat[\label{fig:stepa}]{
\begin{tikzpicture}
\begin{axis}[xlabel={threshold $\tau$}, ylabel= {delay to true change},
    width = 4.3cm,
    ymin = 12,
    ymax = 31,
    scaled y ticks = false,
    cycle list name=yaf,
    yticklabel style={/pgf/number format/fixed},
	no markers,
	xtick = {1, 4, 7, 10},
    ]
\addplot table[x index = 0, y index = 1, header = false, skip first n = 2] {step_threshold.res};
\pgfplotsextra{\yafdrawaxis{1}{10}{12}{31}}
\end{axis}
\end{tikzpicture}}
\subfloat[\label{fig:stepb}]{
\begin{tikzpicture}
\begin{axis}[xlabel={threshold $\tau$}, ylabel= {\# of change points},
    width = 4.3cm,
    ymin = 19,
    ymax = 170,
    scaled y ticks = false,
    cycle list name=yaf,
    yticklabel style={/pgf/number format/fixed},
	no markers,
	xtick = {1, 4, 7, 10},
	ytick = {20, 70, 120, 170}
    ]
\addplot table[x index = 0, y index = 2, header = false, skip first n = 2] {step_threshold.res};
\pgfplotsextra{\yafdrawaxis{1}{10}{19}{170}}
\end{axis}
\end{tikzpicture}}
\subfloat[\label{fig:stepc}]{
\begin{tikzpicture}
\begin{axis}[xlabel={guarantee $\epsilon$}, ylabel= {delay to true change},
    width = 4.3cm,
    xmin = 0,
    xmax = 0.9,
    ymin = 24.0,
    ymax = 27,
    scaled y ticks = false,
    cycle list name=yaf,
    yticklabel style={/pgf/number format/fixed},
	no markers,
	xtick = {0, 0.3, 0.6, 0.9},
    ]
\addplot table[x expr = {1 - \thisrowno{0}}, y index = 1, header = false] {step_approx.res};
\pgfplotsextra{\yafdrawaxis{0}{0.9}{24.0}{27}}
\end{axis}
\end{tikzpicture}}

\caption{Change point detection statistics as a function of threshold parameter $\tau$ and approximation guarantee $\epsilon$ \dtname{Step} data:
(a) average delay for discovering a true change point ($\epsilon = 0$), 
(b) number of discovered change points ($\epsilon = 0$), and
(c) average delay for discovering a true change point ($\tau = 6$). 
Note that in \dtname{Step} there are 19 true change points.
For $\tau = 0.5$, the algorithm had average delay of 1.42 to a true change point
but reported 46\,366 change points (these values are omitted due to scaling
issues).
}

\end{figure}

\textbf{Synthetic sequences:}
We generated 3 synthetic sequences, each of length $200\,000$.
For simplicity we will write $\bernoulli{p}$ to mean a bernoulli random variable
with probability of 1 being $p$.
The first sequence, named \dtname{Ind}, consists of $200\,000$ samples from
$\bernoulli{1/2}$, that is, fair coin flips.  The second sequence, named
\dtname{Step}, consists of $10\,000$ samples from $\bernoulli{1/4}$ followed by
$10\,000$ samples from $\bernoulli{3/4}$, repeated 10 times.  The third
sequence, named \dtname{Slope}, includes 10 segments, each segment consists of
$10\,000$ samples from $\bernoulli{p}$, where $p$ increases linearly from $1/4$
to $3/4$, followed by $10\,000$ samples from $\bernoulli{p}$, where $p$
decreases linearly from $3/4$ to $1/4$.
In addition, we generated 10 sequences, collectively named \dtname{Hill}.  The
length of the sequences varies from $100\,000$ to $1\,000\,000$ with increments
of $100\,000$. Each sequence consists of samples from $\bernoulli{p}$, where
$p$ increases linearly from $1/4$ to $3/4$.

\textbf{Results:}
We start by studying the effect of the threshold parameter $\tau$.  Here, we
used \dtname{Step} sequence; this sequence has 19 true change points.  In
Figure~\ref{fig:stepa}, we show the average delay of discovering the true
change point, that is, how many entries are needed, on average, before a change
is discovered after each true change. In Figure~\ref{fig:stepb}, we also show
how many change points we discovered: ideally we should find only 19 points.
In both experiments we set $\epsilon = 0$. We see from the results that the
delay grows linearly with $\tau$, whereas the number of false change points is
significant for small values of $\tau$ but drop quickly as $\tau$ grows.  For
$\tau = 6$ we detected the ideal 19 change points. We will use this value for
the rest of the experiments.

\begin{figure}[ht!]
\begin{tabular}{rrl}
\subfloat[\label{fig:syntha}]{
\begin{tikzpicture}[baseline=(current bounding box.north)]
\begin{axis}[xlabel={guarantee $\epsilon$}, ylabel= {min approx. ratio},
    width = 4.3cm,
    xmin = 0,
    xmax = 0.9,
    scaled y ticks = false,
    cycle list name=yaf,
    yticklabel style={/pgf/number format/fixed},
	no markers,
	legend entries = {\dtname{Ind}, \dtname{Step}, \dtname{Slope}},
	legend to name = leg:synth,
	xtick = {0, 0.3, 0.6, 0.9},
    ]
\addplot table[x expr = {1 - \thisrowno{0}}, y index = 1, header = false] {ind.res};
\addplot table[x expr = {1 - \thisrowno{0}}, y index = 1, header = false] {step.res};
\addplot table[x expr = {1 - \thisrowno{0}}, y index = 1, header = false] {slope.res};
\pgfplotsextra{\yafdrawaxis{0}{0.9}{0.2}{1}}
\end{axis}
\end{tikzpicture}} &

\subfloat[\label{fig:synthb}]{
\begin{tikzpicture}[baseline=(current bounding box.north)]
\begin{axis}[xlabel={guarantee $\epsilon$}, ylabel= {avg approx. ratio},
    width = 4.3cm,
    xmin = 0,
    xmax = 0.9,
    ymin = 0.97,
    scaled y ticks = false,
    cycle list name=yaf,
    yticklabel style={/pgf/number format/fixed},
	no markers,
	xtick = {0, 0.3, 0.6, 0.9},
    ]
\addplot table[x expr = {1 - \thisrowno{0}}, y index = 2, header = false] {ind.res};
\addplot table[x expr = {1 - \thisrowno{0}}, y index = 2, header = false] {step.res};
\addplot table[x expr = {1 - \thisrowno{0}}, y index = 2, header = false] {slope.res};
\pgfplotsextra{\yafdrawaxis{0}{0.9}{0.97}{1}}
\end{axis}
\end{tikzpicture}}

&
\raisebox{-1.34cm}{\ref{leg:synth}}
\\

\subfloat[\label{fig:synthc}]{
\begin{tikzpicture}
\begin{axis}[xlabel={guarantee $\epsilon$}, ylabel= {$\abs{C} / n$},
    width = 4.3cm,
    xmin = 0,
    xmax = 0.9,
    ymin = 0.0,
    ymax = 0.006,
    cycle list name=yaf,
    yticklabel style={/pgf/number format/fixed},
	no markers,
	tick scale binop=\times,
	xtick = {0, 0.3, 0.6, 0.9},
    ]
\addplot table[x expr = {1 - \thisrowno{0}}, y index = 3, header = false] {ind.res};
\addplot table[x expr = {1 - \thisrowno{0}}, y index = 3, header = false] {step.res};
\addplot table[x expr = {1 - \thisrowno{0}}, y index = 3, header = false] {slope.res};
\pgfplotsextra{\yafdrawaxis{0}{0.9}{0}{0.006}}
\end{axis}
\end{tikzpicture}} &

\subfloat[\label{fig:synthd}]{
\begin{tikzpicture}
\begin{axis}[xlabel={guarantee $\epsilon$}, ylabel= {$\abs{C} / k$},
    width = 4.3cm,
    xmin = 0,
    xmax = 0.9,
    cycle list name=yaf,
    yticklabel style={/pgf/number format/fixed},
	no markers,
	tick scale binop=\times,
	xtick = {0, 0.3, 0.6, 0.9},
    ]
\addplot table[x expr = {1 - \thisrowno{0}}, y index = 4, header = false] {ind.res};
\addplot table[x expr = {1 - \thisrowno{0}}, y index = 4, header = false] {step.res};
\addplot table[x expr = {1 - \thisrowno{0}}, y index = 4, header = false] {slope.res};
\pgfplotsextra{\yafdrawaxis{0}{0.9}{0.38}{1}}
\end{axis}
\end{tikzpicture}} &

\subfloat[\label{fig:synthe}]{
\begin{tikzpicture}
\begin{axis}[xlabel={guarantee $\epsilon$}, ylabel= {running time (s)},
    width = 4.3cm,
    xmin = 0,
    xmax = 0.9,
	ymin = 20,
	ymax = 50,
    cycle list name=yaf,
    yticklabel style={/pgf/number format/fixed},
	no markers,
	tick scale binop=\times,
	xtick = {0, 0.3, 0.6, 0.9},
    ]
\addplot table[x expr = {1 - \thisrowno{0}}, y index = 5, header = false] {ind.res};
\addplot table[x expr = {1 - \thisrowno{0}}, y index = 5, header = false] {step.res};
\addplot table[x expr = {1 - \thisrowno{0}}, y index = 5, header = false] {slope.res};
\pgfplotsextra{\yafdrawaxis{0}{0.9}{20}{50}}
\end{axis}
\end{tikzpicture}}

\end{tabular}

\caption{Performance metrics as a function of approximation guarantee $\epsilon$ on synthetic data. 
Y-axes are as follows:
(a) minimum of ratio $\findchange(B, \epsilon) / \mathit{OPT}$,
(b) average of ratio $\findchange(B, \epsilon) / \mathit{OPT}$,
(c) number of candidates tested / window size (note that $y$-axis is scaled),
(d) number of candidates tested / number of blocks, and
(e) running time in seconds.
}

\end{figure}
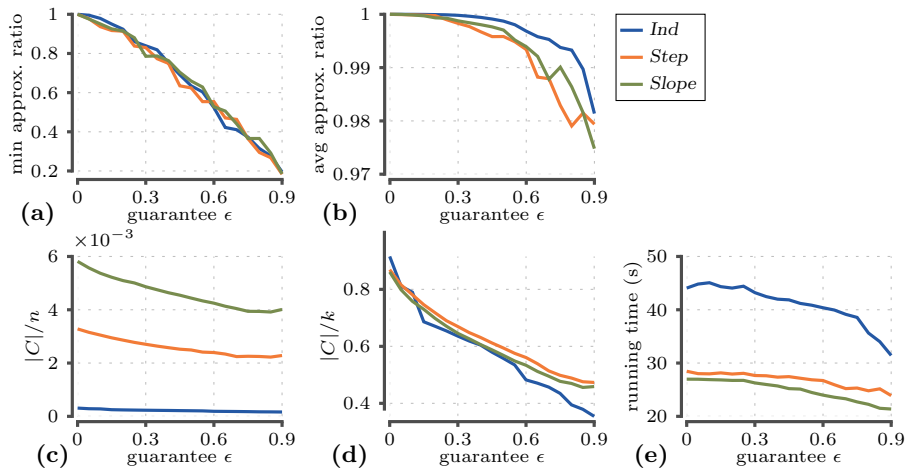

Our next step is to study the quality of the results as a function of $\epsilon$ on synthetic data.
Here we measure the ratio of the scores $g = \findchange(B, \epsilon)$ and
$\mathit{OPT} = \findchange(B, 1)$, that is, the score of the solution to \prbchange.
Note we include all tests, not just the ones that resulted in declaring
a change.
Figure~\ref{fig:syntha} shows the smallest ratio that we encountered as a function of $\epsilon$,
and
Figure~\ref{fig:synthb} shows the average ratio as a function of $\epsilon$.
We see in Figure~\ref{fig:syntha} that the worst case behaves linearly as a function of $\epsilon$.
As guaranteed by Proposition~\ref{prop:correct}, the worst case ratio stays
above $(1 - \epsilon)$. While the worst-case is relatively close
to its theoretical boundary, the average case, shown in Figure~\ref{fig:synthb},
performs significantly better with average ratio being above $0.97$ even for $\epsilon = 0.9$.
The effect of $\epsilon$ on the actual change point detection is demonstrated in Figure~\ref{fig:stepc}.
Since, we may miss the optimal value, the detector becomes more conservative,
which increases the delay for discovering true change. However, the increase is moderate (only about 10\%)
even for $\epsilon = 0.9$.

Our next step is to study speed-up in running time.
Figure~\ref{fig:synthc} shows the number of tests performed compared to $n$, the number of entries
from the last change point as a function of $\epsilon$. We see from the results
that there is significant speed-up when compared to the naive $\bigO{n}$ approach;
the number of needed tests is reduced by 2--3 orders of magnitude. The main reason for this reduction
is due to the border points. Reduction due to using \findblocks is shown in Figure~\ref{fig:synthd}.
Here we see that the number of candidates reduces linearly as a function of $\epsilon$, reducing
the number of candidates roughly by 1/2 for the larger values of $\epsilon$.
The running times (in seconds) are given in Figure~\ref{fig:synthe}. As expected, the running
times are decreasing as a function of $\epsilon$.

\begin{figure}

\subfloat[\label{fig:hilla}]{
\begin{tikzpicture}
\begin{axis}[xlabel={sequence length (in $10^5$)}, ylabel= {running time (m)},
    width = 4.3cm,
    xmin = 1,
    xmax = 10,
	ymin = 0,
	ymax = 18,
    cycle list name=yaf,
    yticklabel style={/pgf/number format/fixed},
	no markers,
	tick scale binop=\times,
	xtick = {1, 4, 7, 10},
	ytick = {0, 6, 12, 18},
    ]
\addplot table[x expr = {\thisrowno{0} / 100000}, y expr = {\thisrowno{2} / 60}, header = false] {hill.res}
	node[pos=0.85, below, font=\scriptsize, color=black, inner sep = 1pt] {$\epsilon = .9$};
\addplot table[x expr = {\thisrowno{0} / 100000}, y expr = {\thisrowno{4} / 60}, header = false] {hill.res}
	node[pos=0.84, above=2pt, font=\scriptsize, color=black, inner sep = 1pt, rotate=10] {$\epsilon = .5$};
\addplot table[x expr = {\thisrowno{0} / 100000}, y expr = {\thisrowno{6} / 60}, header = false] {hill.res}
	node[pos=0.84, above=4pt, font=\scriptsize, color=black, inner sep = 1pt, rotate=30] {$\epsilon = .1$};
\addplot table[x expr = {\thisrowno{0} / 100000}, y expr = {\thisrowno{8} / 60}, header = false] {hill.res}
	node[pos=0.6, above, font=\scriptsize, color=black, inner sep = 1pt, sloped] {$\epsilon = 0$};
\pgfplotsextra{\yafdrawaxis{1}{10}{0}{18}}
\end{axis}
\end{tikzpicture}}
\subfloat[\label{fig:hillb}]{
\begin{tikzpicture}
\begin{axis}[xlabel={sequence length (in $10^5$)}, ylabel= {time / time, $\epsilon = 0$},
    width = 4.3cm,
    xmin = 1,
    xmax = 10,
	ymin = 0,
	ymax = 0.8,
    cycle list name=yaf,
    yticklabel style={/pgf/number format/fixed},
	no markers,
	tick scale binop=\times,
	xtick = {1, 4, 7, 10},
    ]
\addplot table[x expr = {\thisrowno{0} / 100000}, y expr = {\thisrowno{2} / \thisrowno{8}}, header = false] {hill.res}
	node[pos=0.5, below, font=\scriptsize, color=black, inner sep = 1pt, sloped] {$\epsilon = .9$};
\addplot table[x expr = {\thisrowno{0} / 100000}, y expr = {\thisrowno{4} / \thisrowno{8}}, header = false] {hill.res}
	node[pos=0.5, above, font=\scriptsize, color=black, inner sep = 1pt, sloped] {$\epsilon = .5$};
\addplot table[x expr = {\thisrowno{0} / 100000}, y expr = {\thisrowno{6} / \thisrowno{8}}, header = false] {hill.res}
	node[pos=0.5, above, font=\scriptsize, color=black, inner sep = 1pt, sloped] {$\epsilon = .1$};
\pgfplotsextra{\yafdrawaxis{1}{10}{0}{0.8}}
\end{axis}
\end{tikzpicture}}
\subfloat[\label{fig:hillc}]{
\begin{tikzpicture}
\begin{axis}[xlabel={sequence length (in $10^5$)}, ylabel= {$\abs{C} / k$},
    width = 4.3cm,
    xmin = 1,
    xmax = 10,
	ymin = 0,
	ymax = 0.6,
    cycle list name=yaf,
    yticklabel style={/pgf/number format/fixed},
	no markers,
	tick scale binop=\times,
	xtick = {1, 4, 7, 10},
    ]
\addplot table[x expr = {\thisrowno{0} / 100000}, y expr = {\thisrowno{1}}, header = false] {hill.res}
	node[pos=0.5, below, font=\scriptsize, color=black, inner sep = 1pt, sloped] {$\epsilon = .9$};
\addplot table[x expr = {\thisrowno{0} / 100000}, y expr = {\thisrowno{3}}, header = false] {hill.res}
	node[pos=0.45, above, font=\scriptsize, color=black, inner sep = 1pt, sloped] {$\epsilon = .5$};
\addplot table[x expr = {\thisrowno{0} / 100000}, y expr = {\thisrowno{5}}, header = false] {hill.res}
	node[pos=0.53, above=3pt, font=\scriptsize, color=black, inner sep = 1pt, rotate=-30] {$\epsilon = .1$};
\pgfplotsextra{\yafdrawaxis{1}{10}{0}{0.6}}
\end{axis}
\end{tikzpicture}}

\caption{Computational metrics as a function of sequence length for \dtname{Hill} sequences:
(a) running time in minutes,
(b) running time / running time for $\epsilon = 0$, and
(c) number of candidates tested / number of blocks.
Note that $\epsilon = 0$ is equivalent of testing every border index.
}
\label{fig:hill}
\end{figure}
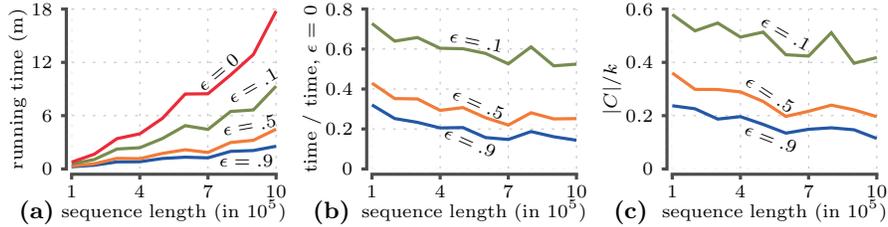

While the main reason for speed-up comes from using border indices, there are
scenarios where using \findblocks becomes significant. This happens when the
number of border indices increases.  We illustrate this effect with
\dtname{Hill} sequences, shown in Figure~\ref{fig:hill}. Here, for the sake of
illustration, we increased the threshold $\tau$ for change point detection so
that at no point we detect change. Having many entries with slowly increasing
probability of 1 yields many border points, which is seen as a fast increase in
running time for $\epsilon = 0$. Moreover, the ratio of candidates tested by 
\findblocks against the number of blocks, as well as the running time, decreases
as the sequence increases in size.

\textbf{Use case with traffic data:}
We applied our change detection algorithm on traffic data, \dtname{network2}, collected
by~\citet{amit19data}. This data contains observed connections between many hosts over
several weeks, grouped in 10 minute periods. We only used data collected during
24.12--29.12 as the surrounding time periods contain a strong hourly artifact.
We then transformed
the collected data into a binary sequence by setting 1 if the connection
was related to SSL, and 0 otherwise. The sequence contains 282\,754 entries grouped in 743 periods of 10 minutes.
Our algorithm ($\epsilon = 0$, $\tau = 6$)
found 12 change points, shown in Figure~\ref{fig:network}.
These patterns show short bursts of non-SSL connections.
One exception is the change after the index 300, where the previously high SSL activity
is resolved to a normal behavior.

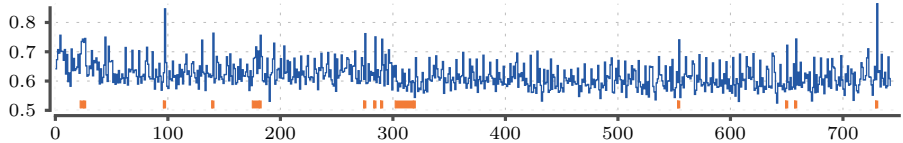
\begin{figure}
\begin{tikzpicture}
\begin{axis}[
    width = 12.8cm,
	height = 3cm,
    ymin = 0.5,
    xmax = 750,
    cycle list name=yaf,
    yticklabel style={/pgf/number format/fixed},
	tick scale binop=\times,
    ]
\addplot+[const plot, every axis plot post/.append style= {line width=0.5pt}, no markers]
	table[x expr = \coordindex, y expr = {\thisrowno{1} / (\thisrowno{1} + \thisrowno{2})}, header = false] {network_ssl.dat};

\addplot+[
		only marks,
		mark options = {
			mark = |,
			mark size = 1.5pt
		},
		error bars/.cd,
			x dir = minus,
			x explicit,
			error bar style={line width=3pt},
			error mark options = {mark size=1.5pt, rotate=90}
	]
	table[x index = 1, y expr = 0.52, x error expr = {\thisrowno{1} - \thisrowno{3} + 1}, header = false] {network_change.dat};

\pgfplotsextra{\yafdrawaxis{0}{750}{0.5}{0.85}}
\end{axis}
\end{tikzpicture}
\caption{Proportion of non-SSL connections in \dtname{Network2} traffic data over time, in 10 minute periods. The bars indicate the change points:
the end of the bar indicates when change was discovered and the beginning of the bar indicate the optimal split.}
\label{fig:network}
\end{figure}

\section{Conclusions}\label{sec:conclusions}
In this paper we presented a change point detection approach for binary streams
based on finding a split in a current window optimizing a likelihood ratio.
Finding the optimal split needs $\bigO{n}$ time, so in order for this approach to be
practical, we introduced an approximation scheme that yields $(1 - \epsilon)$
approximation in $\bigO{\epsilon^{-1} \log^2 n}$. The scheme is implemented by using border points, an idea
adopted from segmentation of log-linear models, and then further reducing
the candidates by ignoring indices that border similar blocks. 

Most of the time the number of borders will be small, and the additional
pruning is only required when the number of borders start to increase. This
suggests that a hybrid approach is sensible: we will iterate over borders if
there are only few of them, and switch to approximation technique only when
the number of borders increase.

We should point that even though the running time is poly-logarithmic, the
space requirement is at worst $\bigO{n^{2/3}}$.  This can be rectified by
simply removing older border points but such removal may lead to a suboptimal
answer.  An interesting direction for a future work is to study how to reduce
the space complexity without sacrificing the approximation guarantee.

In this paper, we focused only on binary streams. Same concept has the
potential to work also on other type of data types, such as integers or
real-values. The bottleneck here is Proposition~\ref{prop:candsize} as
it relies on the fact that the underlying stream is binary. We will leave
adopting these results to other data types as a future work.

\bibliographystyle{splncsnat}
\bibliography{bibliography}

\begin{thebibliography}{21}
\providecommand{\natexlab}[1]{#1}
\providecommand{\url}[1]{\texttt{#1}}
\providecommand{\urlprefix}{}

\bibitem[{Adams and MacKay(2007)}]{Adams07}
Adams, R.P., MacKay, D.J.: Bayesian online changepoint detection.
\newblock Technical report, University of Cambridge, Cambridge, UK (2007)

\bibitem[{Aminikhanghahi and Cook(2017)}]{Aminikhanghahi2017}
Aminikhanghahi, S., Cook, D.J.: A survey of methods for time series change
  point detection.
\newblock Knowledge and Information Systems 51(2), 339--367 (May 2017)

\bibitem[{Amit et~al.(2019)Amit, Matherly, Hewlett, Xu, Meshi, and
  Weinberger}]{amit19data}
Amit, I., Matherly, J., Hewlett, W., Xu, Z., Meshi, Y., Weinberger, Y.: Machine
  learning in cyber-security --- problems, challenges and data sets.
\newblock In: The AAAI-19 Workshop on Engineering Dependable and Secure Machine
  Learning Systems (2019)

\bibitem[{Baena-Garcia et~al.(2006)Baena-Garcia, Campo-Avila, Fidalgo, Bifet,
  Gavalda, and Morales-Bueno}]{Baena06}
Baena-Garcia, M., Campo-Avila, J.D., Fidalgo, R., Bifet, A., Gavalda, R.,
  Morales-Bueno, R.: Early drift detection method.
\newblock In: In 4th Int. Workshop on Knowledge Discovery from Data Streams
  (2006)

\bibitem[{Basseville and Nikiforov(1993)}]{Basseville93}
Basseville, M., Nikiforov, I.V.: Detection of Abrupt Changes -- Theory and
  Application.
\newblock Prentice-Hall (1993)

\bibitem[{Bellman(1961)}]{bellman:61:on}
Bellman, R.: On the approximation of curves by line segments using dynamic
  programming.
\newblock Communications of the ACM 4(6), 284--284 (1961)

\bibitem[{Bifet and Gavalda(2007)}]{Bifet07}
Bifet, A., Gavalda, R.: Learning from time-changing data with adaptive
  windowing.
\newblock In: In SIAM Int. Conf. on Data Mining. pp. 443--448 (2007)

\bibitem[{Calders et~al.(2008)Calders, Dexters, and Goethals}]{calders2008pava}
Calders, T., Dexters, N., Goethals, B.: Mining frequent items in a stream using
  flexible windows.
\newblock Intell. Data Anal. 12(3), 293--304 (2008)

\bibitem[{Dries and R\"{u}ckert(2009)}]{Dries09}
Dries, A., R\"{u}ckert, U.: Adaptive concept drift detection.
\newblock Stat. Anal. Data Min. 2(5--6), 311--327 (2009)

\bibitem[{Gama et~al.(2004)Gama, Medas, Castillo, and Rodrigues}]{Gama04}
Gama, J., Medas, P., Castillo, G., Rodrigues, P.: Learning with drift
  detection.
\newblock In: In SBIA Brazilian Symposium on Artificial Intelligence. pp.
  286--295 (2004)

\bibitem[{Guha et~al.(2006)Guha, Koudas, and Shim}]{guha:06:estimate}
Guha, S., Koudas, N., Shim, K.: Approximation and streaming algorithms for
  histogram construction problems.
\newblock ACM Transactions of Database Systems 31(1), 396--438 (2006)

\bibitem[{Guha and Shim(2007)}]{guha:07:linear}
Guha, S., Shim, K.: A note on linear time algorithms for maximum error
  histograms.
\newblock {IEEE} Transactions on Knowledge and Data Engineering 19(7), 993--997
  (2007)

\bibitem[{Harel et~al.(2014)Harel, Mannor, El-Yaniv, and Crammer}]{Harel14}
Harel, M., Mannor, S., El-Yaniv, R., Crammer, K.: Concept drift detection
  through resampling.
\newblock In: Proc. of the 31st Int. Conf. on Machine Learning. pp. 1009--1017.
  ICML (2014)

\bibitem[{Kawahara and Sugiyama(2012)}]{Kawahara12}
Kawahara, Y., Sugiyama, M.: Sequential change-point detection based on direct
  density-ratio estimation.
\newblock Statistical Analysis and Data Mining 5, 114--127 (2012)

\bibitem[{Kifer et~al.(2004)Kifer, Ben-David, and Gehrke}]{Kifer04}
Kifer, D., Ben-David, S., Gehrke, J.: Detecting change in data streams.
\newblock In: Proc. of the 13th Int. Conf. on Very Large Data Bases. pp.
  180--191. VLDB (2004)

\bibitem[{de~Leeuw et~al.(2009)de~Leeuw, Hornik, and Mair}]{leeuw2009isotonic}
de~Leeuw, J., Hornik, K., Mair, P.: Isotone optimization in r:
  Pool-adjacent-violators algorithm (pava) and active set methods.
\newblock Journal of Statistical Software, Articles 32(5), 1--24 (2009)

\bibitem[{Nishida and Yamauchi(2007)}]{Nishida07}
Nishida, K., Yamauchi, K.: Detecting concept drift using statistical testing.
\newblock In: Proc. of the 10th Int. Conf. on Discovery Science. pp. 264--269
  (2007)

\bibitem[{Ross et~al.(2012)Ross, Adams, Tasoulis, and Hand}]{Ross12}
Ross, G.J., Adams, N.M., Tasoulis, D.K., Hand, D.J.: Exponentially weighted
  moving average charts for detecting concept drift.
\newblock Pattern Recognition Letters 33(2), 191--198 (2012)

\bibitem[{Tatti(2013)}]{tatti2013fast}
Tatti, N.: Fast sequence segmentation using log-linear models.
\newblock Data mining and knowledge discovery 27(3), 421--441 (2013)

\bibitem[{Tatti(2019)}]{tatti2019segmentation}
Tatti, N.: Strongly polynomial efficient approximation scheme for segmentation.
\newblock Inf. Process. Lett. 142, 1--8 (2019),
  \urlprefix\url{https://doi.org/10.1016/j.ipl.2018.09.007}

\bibitem[{Terzi and Tsaparas(2006)}]{terzi:06:efficient}
Terzi, E., Tsaparas, P.: Efficient algorithms for sequence segmentation.
\newblock In: Proceedings of the 6th {SIAM} International Conference on Data
  Mining (SDM). pp. 316--327 (2006)

\end{thebibliography}

\appendix
\section{Proof of Proposition~\ref{prop:block}}

Before proving the proposition, we need some additional notation. Let us write
\[
	\gain{i; p_1, p_2} = \lhood{a_1, b_1; p_1} + \lhood{a_2, b_2; p_2}\quad.
\]

If $p_1 = a_1 / (a_1 + b_1)$ and $p_2 = a_2 / (a_2 + b_2)$, we will drop them
from notation and write instead $\gain{i}$. Note that $\gain{i} \geq \gain{i;
r_1, r_2}$ for any $r_1$ and $r_2$.

\begin{proof}
Let $i$ be a solution for \prbchangeinc. In case of ties, let $i$ be the smallest
index producing the optimal solution.

Let $p_1 \leq p_2$ be the corresponding parameters for the two bernoulli
variables, $\gain{i} = \gain{i; p_1, p_2}$.

Assume that $p_1 = 0$. If $s_i = 0$, then it is easy to show that 
\[
	\gain{i} \leq \gain{i + 1; p_1, p_2} \leq \gain{i + 1}\quad.
\]
By repeating this argument, we can show that there is $i'$ such that $\gain{i'}
\geq \gain{i}$. Moreover, $s_{i'} = 1$ and $s_j = 0$ for any $j < i$ (it is safe to assume that $S$ has non-zero values).  This
makes $i'$ a border index. Thus we can safely assume that $p_1 > 0$, and
similarly $p_2 < 1$.

If $i$ is a border
index, then we are done. Assume that $i$ not a border index, that is, there are
two integers $x < i < y$ such that
\[
	\frac{1}{i - x} \sum_{j = x}^{j - 1} s_j \geq \frac{1}{y - i} \sum_{j = i}^{y - 1} s_j\quad.
\]
We will denote the left hand-side of the inequality $d_1$ and the right
hand-side with $d_2$.

Let us define 
\[
z_1 = \log (1 - p_1),\ 
z_2 = \log (1 - p_2),\ 
\alpha_1 = \log \frac{p_1}{1 - p_1},\ \text{and}\ 
\alpha_2 = \log \frac{p_2}{1 - p_2}\quad.
\]

We can now write $\gain{k; p_1, p_2}$ as
\[
	\gain{k; p_1, p_2} = (k - 1) z_1 + \alpha_1 \sum_{j = 1}^{k - 1} s_j + (n - k + 1) z_2 + \alpha_2 \sum_{j = k}^{n} s_j,
\]
for any index $k$.
We can now write the difference between the two scores as
\[
	\gain{y; p_1, p_2} - \gain{i; p_1, p_2} = (z_1 - z_2) (y - i) + (\alpha_1 - \alpha_2) \sum_{j = i}^{y - 1} s_j\quad.
\]
Normalizing this difference with $y - i$ leads to
\[
	\frac{1}{y - i}\pr{\gain{y; p_1, p_2} - \gain{i; p_1, p_2}} = (z_1 - z_2) + (\alpha_1 - \alpha_2) d_2\quad.
\]
Similarly,
\[
	\frac{1}{i - x}\pr{\gain{i; p_1, p_2} - \gain{x; p_1, p_2}} = (z_1 - z_2) + (\alpha_1 - \alpha_2) d_1\quad.
\]

Since $\gain{i}$ is optimal, $\gain{i} \geq \gain{y}$.
Then
\[
\begin{split}
	\frac{1}{i - x}\pr{\gain{i; p_1, p_2} - \gain{x; p_1, p_2}} & = (z_1 - z_2) + (\alpha_1 - \alpha_2)d_1 \\
	&\leq (z_1 - z_2) + (\alpha_1 - \alpha_2) d_2  \\
	& = \frac{1}{y - i}\pr{\gain{y; p_1, p_2} - \gain{i; p_1, p_2}} \\
	& \leq \frac{1}{y - i}\pr{\gain{y} - \gain{i}} \leq 0 \quad. \\
\end{split}
\]
Here we used the fact that
$\alpha_1 - \alpha_2 \leq 0$ since $p_1 \leq p_2$ and, by definition, $d_2 \leq d_1$.
In other words, $\gain{x} \geq \gain{x; p_1, p_2} \geq \gain{i; p_1, p_2} = \gain{i}$.
This violates the minimality of $i$, hence $i$ must be a border index. This completes the proof.\qed
\end{proof}

\end{document}